\documentclass{article}

    \PassOptionsToPackage{numbers, compress}{natbib}
 \usepackage[preprint]{neurips_2026}

\usepackage[utf8]{inputenc} 
\usepackage[T1]{fontenc}    
\usepackage{hyperref}       
\usepackage{url}            
\usepackage{booktabs}       
\usepackage{amsfonts}       
\usepackage{nicefrac}       
\usepackage{microtype}      
\usepackage{xcolor}         
\usepackage{subfigure}
\usepackage{amsmath}
\usepackage{amssymb}
\usepackage{mathtools}
\usepackage{amsthm}
\usepackage{graphicx}
\usepackage{afterpage}
\usepackage{subcaption}   
\usepackage{multirow}
\usepackage{wrapfig}
\usepackage{arydshln}   
\usepackage{makecell}   
\usepackage[table,dvipsnames]{xcolor}

\usepackage{tcolorbox}
\tcbuselibrary{listings}
\usepackage{thmtools}
\usepackage{tikz}
\usepackage{enumitem}
\usetikzlibrary{shapes,arrows,calc,positioning}
\usetikzlibrary{graphs}
\usepackage{bm}
\definecolor{darkgreen}{rgb}{0, 0.5, 0}

\usepackage[capitalize,noabbrev]{cleveref}

\DeclareMathOperator*{\argmax}{arg\,max}

\theoremstyle{plain}

\theoremstyle{definition}

\theoremstyle{remark}

\newenvironment{proofsketch}{%
  \begin{proof}[Proof sketch]
}{%
  \end{proof}
}
\title{Conformal Agent Error Attribution}

\author{%
  Naihe Feng\thanks{Equal Contribution} \\
  Dalhousie University\\
  Halifax, NS, Canada \\
  \texttt{nfeng@dal.ca} \\
\And
    Yi Sui$^*$ \\
    Layer 6 AI\\
    Toronto, ON, Canada\\
    \texttt{amy@layer6.ai}\\
    \And
    Shiyi Hou \\
    Layer 6 AI\\
    Toronto, ON, Canada\\
    \texttt{gloria@layer6.ai}\\
  \AND
  Ga Wu\\
  Dalhousie University\\
  Halifax, NS, Canada \\
  \texttt{ga.wu@dal.ca} \\
  \And
  Jesse C. Cresswell \\
  Layer 6 AI \\
  Toronto, ON, Canada\\
  \texttt{jesse@layer6.ai}
}

\begin{document}

\maketitle

\begin{abstract}

When multi-agent systems (MAS) fail, identifying where the decisive error occurred is the first step for automated recovery to an earlier state. Error attribution remains a fundamental challenge due to the long interaction traces that large language model-based MAS generate. This paper presents a framework for error attribution based on conformal prediction (CP) which provides finite-sample, distribution-free coverage guarantees. We introduce new algorithms for filtration-based CP designed for sequential data such as agent trajectories. Unlike existing CP algorithms, our approach predicts sets that are contiguous sequences to enable efficient recovery and debugging. We verify our theoretical guarantees on a variety of agents and datasets, show that errors can be precisely isolated, then use prediction sets to rollback MAS to correct their own errors. Our overall approach is model-agnostic, and offers a principled uncertainty layer for MAS error attribution. We release code at \href{https://github.com/layer6ai-labs/conformal-agent-error-attribution}{\texttt{github.com/layer6ai-labs/conformal-agent-error-attribution}}.
\end{abstract}

\section{Introduction}
\label{sec:intro}

\begin{wrapfigure}[7
]{r}{0.5\textwidth}
\vspace{-12pt}
    \centering
    \includegraphics[width=0.48\textwidth, trim={25 77 31 6
    }, clip]{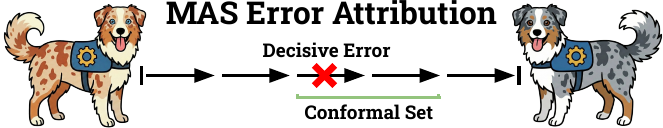}
    \vspace{-5pt}
    \caption{\footnotesize Conformal agent error attribution isolates the decisive error in a failed MAS trajectory within a conformal prediction set, providing statistical guarantees of coverage.}
    \label{fig:MAS_error_attribution}
\end{wrapfigure}
\normalsize

Advances in large language models (LLMs) have driven the widespread adoption of multi-agent systems (MAS) for complex tasks requiring decomposition, coordination, and tool use~\cite{guo2024largelanguagemodelbased}, with strong empirical performance in domains such as software engineering~\cite{hong2024metagpt, huang2024agentcoder}, scientific discovery~\cite{Ghafarollahi2024SciAgents}, and financial decision making~\cite{yu2024fincon}. However, the increased system complexity and rich interactions in MAS make them prone to errors from incorrect intermediate decisions, miscoordination among agents, and long-horizon dependencies~\cite{pan2025why, leung2026classifying}. While detecting \emph{overall} task failure is often straightforward, understanding \textit{why and where} a failure originated remains challenging yet critical for debugging, and self-correction. Identifying the decision step that constitutes the decisive error point has emerged as a central challenge for improving MAS.

Most existing MAS error attribution approaches, including naive LLM-as-a-judge methods~\cite{zhang2025agent}, structured reasoning pipelines~\cite{hengst2025hierarchical, zhu2025raffles, yu2025correct}, and fine-tuned attribution models~\cite{kong2025aegis}, ultimately produce a point prediction, committing to a single responsible step. In practice, point predictions provide limited actionable insight for practitioners as they offer no principled form of uncertainty quantification to assess reliability, undermining the trustworthiness of error attribution systems~\cite{ross2026textual}. Conformal prediction (CP) offers a promising direction for addressing this limitation through the generation of \emph{prediction sets}. CP enables reliable decision-making under uncertainty by providing statistical guarantees across a range of applications~\cite{angelopoulos2022gentle, cresswell2024}.

Motivated by these developments, we propose an uncertainty-aware framework for error attribution in MAS based on CP, which provides finite-sample, distribution-free coverage guarantees. Rather than predicting a single step, our methods identify a localized region of the execution trace that is guaranteed to contain the decisive error at a user-specified confidence level (see \Cref{fig:MAS_error_attribution}). We introduce novel methods for filtration-based CP which are adapted to sequential data structures, like agent trajectories. Unlike existing CP approaches that produce arbitrary sets, our methods produce \emph{contiguous} sets, aligning with the inherent ordinal structure of sequential data. Finally, we use conformal sets to rollback the MAS to before the decisive error, allowing the agent to restart and fix its own mistakes. Our approach is model-agnostic and can wrap existing black-box attribution scores, while providing a principled uncertainty layer for error attribution in real-world MAS.

\section{Background \& Related Work}
\label{sec:background}

\subsection{Post-hoc Multi-agent System Error Attribution}\label{subsec:mas-error-attribution}

Recent work on error attribution in MAS has primarily studied post-hoc localization of errors using execution traces. Early approaches used \textit{naive LLM-as-a-judge}, where a single LLM directly predicts the responsible step from a failed trace~\cite{zhang2025agent}. Subsequent work improved attribution quality through more sophisticated LLM pipelines, including \emph{context engineering} as well as multi-LLM frameworks. For example, ECHO~\cite{banerjee2025hierarchical} improves attribution by organizing long traces into hierarchical contexts and aggregating evaluations via consensus, while RAFFLES~\cite{zhu2025raffles} employs a multi-turn, multi-LLM architecture that iteratively proposes and critiques candidate error steps. Additionally, CORRECT~\cite{yu2025correct} incorporates retrieval to localize the error step based on similar events. A complementary line of work \textit{fine-tunes specialized LLM judges} for error attribution. In particular, AEGIS~\cite{kong2025aegis} constructs large-scale labeled failure traces via controlled error injection for fine-tuning LLMs on the task.

In our experiments we compare the efficacy of these three main classes of evaluators: naive, context-engineered, and fine-tuned LLM judges. We note that all of the above research assumes a single decisive error, whereas in practical MAS applications small errors may accumulate into large ones. Due to the lack of labeled datasets with more nuanced error definitions, we follow current work and focus on the decisive error setting.
\subsection{Conformal Prediction}\label{subsec:conformal}
For a classification problem where inputs $x\in \mathcal{X}$ and ground-truth values $y^*\in \mathcal{Y}=[\ell]:=(1, \dots, \ell)$ are drawn jointly from a distribution $(x, y^*)\sim \mathbb{P}$, CP first calibrates a threshold $\hat q$ from a set of held-out data. Then for a new datapoint $x_{n+1}$, CP outputs a set of classes $C(x_{n+1}; \hat q)\subseteq \mathcal{Y}$ which contains $y^*$ with high probability $\mathbb{P}[y_{n+1}^* \in C(x_{n+1};\hat q)] \geq 1 - \alpha$. This \emph{coverage} guarantee is distribution-free and valid in finite samples, while also allowing the user to set their own error tolerance $\alpha$ \citep{vovk2005algorithmic, shafer2008tutorial}.

To perform CP, one first defines a \emph{conformal score} function $S:\mathcal{X} \times \mathcal{Y}\to \mathbb{R}^+$, which should take smaller values when $y=y^*$ is the correct label for $x$. In practice, $S(x, y)$ often leverages the predictions of a pre-trained classification model $f : \mathcal{X}\to \mathcal{Y}$. Using a set of $n$ calibration datapoints, CP computes the scores $\{S_i\}_{i=1}^{n} = \{S(x_i, y^*_i)\}_{i=1}^{n}$, and finds the $\tfrac{\lceil{(n+1)(1-\alpha)}\rceil}{n}$ quantile which is set as the threshold $\hat q$. Then prediction sets can be generated by including all classes for which the score is greater than $\hat q$, $C(x_{n+1}; \hat q)  = \{y\in \mathcal{Y} \mid S(x_{n+1}, y) \leq \hat q \}$. When $x_{n+1}$ is exchangeable with the calibration data, the coverage guarantee is valid. Exchangeability is a mild assumption that automatically holds when data is i.i.d., and hence is reasonable for many machine learning contexts, including the agent error attribution task as we show in our experiments. At equal coverage levels $1-\alpha$, smaller prediction sets are considered more useful both for uncertainty quantification over the predictions of $f_\theta$ \cite{romano2020aps, angelopoulos2021raps, huang2024saps}, and in downstream tasks \cite{cresswell2024, cresswell2025conformal}.
\subsection{Conformal Prediction for Sequential Data}\label{subsec:sequential}
Another common setting is where data is sequential, $x=(c_1, \dots, c_{\ell})$, with variable length $\ell$, where the ground truth $y^*\subset x$ is a subset of elements. Following \citet{kuwahara2025document}, the principles of CP can be used to calibrate a threshold $\hat q$, and predict a subset $C(x_{n+1}; \hat q)\subseteq x_{n+1}$ which retains the ground truth elements with high probability, $\mathbb{P}[y_{{n+1}}^* \subseteq C(x_{n+1}; \hat q)] \geq 1 - \alpha$. In some settings, $y^*$ will consist of multiple elements, and the predicted set $C(x_{n+1}; \hat q)$ need not be contiguous. For agent error attribution we take $x$ to be the agent's trajectory, and $y^*$ to be the single decisive error---one of the $c_i$. As we will discuss, for downstream applications including automated rollbacks of the agent's state it is desirable to predict sets of \emph{consecutive} elements, rather than arbitrary subsets.  Hence, we develop novel CP algorithms that satisfy a coverage guarantee using contiguous prediction sets,
\begin{equation}\label{eq:coverage-guarantee-contiguous}
  \mathbb{P}[y_{n+1}^*\in C(x_{n+1};\hat q)] \geq 1 - \alpha,  \quad C(x_{n+1}; \hat q) = (c_j, ..., c_k).
\end{equation}

The only existing CP algorithms that produce contiguous sets were designed for hierarchical classification \cite{mortier2025conformal}. We describe one such algorithm in \Cref{sec:crsvp} and adapt it for sequential data.

\section{Conformal Agent Error Attribution}
\label{sec:method}

For the remainder of this work we take $x=(c_1, \dots, c_{\ell})$ to be an agent trajectory which has failed to complete the desired task. Each step $c_j$ can contain any available information such as the environment's state, action taken, and observed response. One of the steps $y^*\in x$ is labeled as the decisive error—the earliest error that the MAS cannot recover from. The aim is to produce a prediction set $C(x_{n+1})\subseteq x_{n+1}$ that gives valid coverage, where smaller sets are preferred.

Applying CP to agent error attribution requires two components: an algorithm which takes a calibration dataset and generates a prediction set for $x_{n+1}$ with valid coverage; and a scoring function $g(C(x))$ acting on sets of steps, which quantifies the likelihood that $y^*\in C(x)$. We design these two components separately so that they are interchangeable, and discuss the pros and cons of each option.
\subsection{Conformal Algorithms for Agent Error Attribution}
\subsubsection{Vanilla Conformal Prediction}
The simplest approach is to ignore the sequential nature of $x$ and treat all steps as unordered classes in an $\ell$-way classification task. We will write $S_\text{VCP}=1-g$ for the conformal score function, and $C_\text{VCP}(x_{n+1}; \hat q) = \{c_i\in x_{n+1} \mid S_\text{VCP}(x_{n+1}, c_i) \leq \hat q \}$ for prediction sets generated by Vanilla CP (VCP). Prediction requires iterating over every step in the trajectory using $\ell$ evaluations of $g$, and does not produce contiguous sets.
\subsubsection{Leaf-to-Root Tree Traversal}\label{sec:crsvp}
\begin{wrapfigure}[10]{r}{0.48\textwidth}
  \centering 
  \vspace{-8pt}
  \hspace{-12pt}
  \begin{tikzpicture}[sibling distance=2.2em,
      every node/.style = {align=center},
      line/.style={draw, -latex'},
      edge from parent/.style={draw,-latex',font=\normalsize},
      level 1/.style={sibling distance=34mm, level distance=8mm, font=\normalsize}, 
      level 2/.style={sibling distance=16mm, level distance=8mm,text width=1.8cm,font=\Large}] 
      \node[font=\normalsize] {$v_1=(c_1,c_2,c_3,c_4)$}
        child { node (ch1){$v_2=(c_1,c_2)$} 
            child { node {$\substack{v_4=(c_1)}$} }
            child { node {$\substack{v_5=(c_2)}$} }
            }
        child { node (ch2){$v_3=(c_3,c_4)$} 
            child { node {$\substack{v_{6}=(c_3)}$} }
            child { node {$\substack{v_{7}=(c_4)}$} } };
\end{tikzpicture}
  \caption{\footnotesize An example binary tree $\mathcal{T}$ representing an agent trajectory $x$ consisting of four steps $c_1, ..., c_4$. Contiguous prediction sets $C(x_{n+1})$ will consist of a single node $v_i$.}
  \label{fig:tree-example}
\end{wrapfigure}
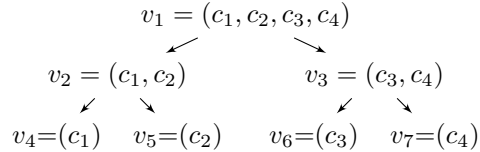

To produce contiguous sets, we can adapt algorithms for hierarchical classification by mapping agent trajectories $x$ onto a binary tree $\mathcal{T}$ as depicted in \Cref{fig:tree-example}, with root node $v_1 = [\ell]$, and leaf nodes $v_\ell, \dots, v_{2\ell-1}$ as individual steps $c_1, \dots, c_\ell$. CP is conducted by traversing the tree from leaf to root following the CRSVP algorithm \citep{mortier2025conformal} described in full detail in \Cref{app:b}. For each test datapoint, CRSVP outputs one node of the tree as the prediction set which is always a contiguous set, and guarantees the lower bound on coverage in \Cref{eq:coverage-guarantee-contiguous}.

CRSVP lacks an upper bound on coverage, uses $\ell$ evaluations of $g$ for prediction, and produces inflexible sets following the tree's splits. For example, in \Cref{fig:tree-example} the middle steps $c_2$ and $c_3$ can only be predicted together in the trivial case where all steps are predicted ($v_1$). VCP and CRSVP serve as baselines in our experiments. The following novel algorithms improve on their limitations.
\newpage
\subsubsection{Left (Right) Filtration}\label{sec:LF}

\begin{wrapfigure}[9]{r}{0.48\textwidth}
  \centering 
    \includegraphics[width=0.95\linewidth, trim={0 0 35 0}, clip]{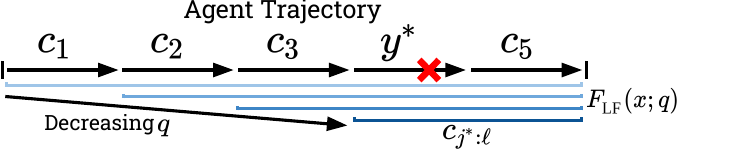}
    \caption{\footnotesize Left filtering with $F_\text{LF}(x;q)$ progressively removes steps from the left as $q$ decreases. The smallest $q$ which retains the decisive error $y^*$  is used as the conformal score $S_\text{LF}(x, y^*)$.}
    \label{fig:LF}
\end{wrapfigure}
Viewing a trajectory $x=(c_1, \dots, c_{\ell})$ as a sequence, Left Filtration (LF) progressively removes steps from $x$ starting on the left with $c_{1}$ until the remaining subsequence scores below a calibrated threshold, returning a \emph{suffix} of the full sequence.

We assume access to a scoring function $g_{\text{LF}}$ which scores subintervals $c_{j: k} := (c_j, \dots, c_k) \subseteq x$ for their likelihood to contain $y^*$, imposing the boundary conditions $g_{\text{LF}}(\varnothing) = 0$, since the empty interval cannot contain $y^*$, and $g_{\text{LF}}(x) = \infty$ since $y^*\in x$. We will use $j^*$ as the index where $y^*$ appears, so $c_{j^*} = y^*$. Finally, it is desirable, but not mandatory, to have $g_{\text{LF}}$ obey a monotonicity condition:
\begin{equation}\label{eq:g_LF_nesting}
    c_{b: c}\subseteq c_{a: d} \implies g_{\text{LF}}(c_{b: c}) \leq g_{\text{LF}}(c_{a: d}).
\end{equation}
Next, we define a filtering function which returns the longest suffix that has a low enough score $g_{\text{LF}}$. Formally, we write the set of suffixes as 
$\mathcal{I}_{\text{LF}}:= \{c_{j:  \ell}\}_{j=1}^{\ell+1}$, with the conventions for subintervals that $c_{j: j} = c_j$, and $c_{j: k} = \varnothing$ when $j>k$. With these conventions, the left-filtering function is
\begin{equation}\label{eq:filtering_left}
    F_{\text{LF}}(x; q) := \argmax_{c_{j: \ell} \in \mathcal{I}_{\text{LF}}} \big(\vert c_{j: \ell}\vert \mid g_{\text{LF}}(c_{j: \ell}) \leq q\big),
\end{equation}
where $q\in \mathbb{R}^+$, and we note that $F_{\text{LF}}(x; \infty) = x$. Since the suffixes in $\mathcal{I}_{\text{LF}}$ are nested, $F_{\text{LF}}(x; q)$ also satisfies a nesting property demonstrating that more steps are filtered out as $q$ decreases. (All formal proofs are given in \Cref{app:lf-proof}.)
\begin{restatable}{lemma}{LFNesting}
\label{thm:LFNesting}
For any $x$ and thresholds $0\leq q_1 \leq q_2$, $F_{\text{LF}}(x; q_1) \subseteq F_{\text{LF}}(x; q_2)$.
\end{restatable}
\vspace{-10pt}
\begin{proofsketch} Suffixes themselves are nested. When $q_1$ increases to $q_2$, the set of valid suffixes with $g_{\text{LF}}(c_{j:\ell}) \leq q$ can only grow, and $F_{\text{LF}}$ returns the largest valid suffix.
\end{proofsketch}
For a trajectory $x$, the conformal score is effectively the smallest threshold $q$ where $y^*$ is not filtered out (see \Cref{fig:LF}). More formally we define
\begin{equation}\label{eq:left-filter-score}
    S_{\text{LF}}(x, y^*) := \inf\big(q\in \mathbb{R^+} \mid y^*\in F_{\text{LF}}(x; q)\big).
\end{equation}
Although this definition involves two optimizations, $S_{\text{LF}}$ is designed so that its computation greatly simplifies in practice. When $g_\text{LF}$ obeys monotonicity (\Cref{eq:g_LF_nesting}), $S_{\text{LF}}$ is simply the value of $g_{\text{LF}}$ on the suffix that starts at $y^*$.
\begin{restatable}{proposition}{LFComputation}
\label{thm:LFComputation}
When $g_\text{LF}$ obeys monotonicity such that $c_{b: c}\subseteq c_{a: d} \implies g_{\text{LF}}(c_{b: c}) \leq g_{\text{LF}}(c_{a: d})$, then $S_{\text{LF}}(x, y^*) = g_\text{LF}(c_{j^*:\ell})$.
\end{restatable}
\begin{proofsketch}
$S_{\text{LF}}(x, y^*)$ optimizes over $q$ where $F_{\text{LF}}(x; q)$ returns a suffix at least as large as $c_{j^*:\ell}$. Due to monotonicity, suffixes larger than $c_{j^*:\ell}$ cannot have scores less than $g_\text{LF}(c_{j^*:\ell})$, so $g_\text{LF}(c_{j^*:\ell})$ is the smallest valid threshold.
\end{proofsketch}
The LF conformal algorithm computes $S_{\text{LF}}$ for each calibration datapoint, and sets the conformal threshold $\hat q$ as the $\frac{\lceil(n+1)(1-\alpha)\rceil}{n}$ quantile. For a test datapoint we predict $C_{\text{LF}}(x_{n+1}; \hat q) = F_{\text{LF}}(x_{n+1}; \hat q)$, which is the longest suffix $c_{ j:\ell}$ such that $g_{\text{LF}}(c_{j:\ell}) \leq \hat q$. A shorter suffix is preferable since it better isolates the decisive error. This algorithm satisfies \Cref{eq:coverage-guarantee-contiguous} for any scoring function $g_{\text{LF}}$ as defined above. Before proving this coverage guarantee, we present a lemma to assist:
\begin{restatable}{lemma}{LFEquivalence}
\label{thm:LFEquivalence}
For a fixed $\hat q$, we have $S_{\text{LF}}(x, y^*) \leq \hat{q} \iff y^* \in F_{\text{LF}}(x; \hat{q})$.
\end{restatable}
\begin{proofsketch}
The infimum in $S_{\text{LF}}(x, y^*)$ is always achieved at $q_\text{min}=\min \{g_{\text{LF}}(c_{j:\ell}) \mid j \le j^*\}$, and $y^* \in F_{\text{LF}}(x; q_\text{min})$. Because $F_{\text{LF}}$ is nested in $q$ (\Cref{thm:LFNesting}), increasing $q_\text{min}$ to $\hat q$ maintains $y^* \in F_{\text{LF}}(x; \hat q)$.
\end{proofsketch}
With these facts established, we can prove the coverage guarantee of \Cref{eq:coverage-guarantee-contiguous} using a standard conformal argument.
\begin{restatable}{theorem}{LFTheorem}
\label{thm:LF}
Suppose $\{(x_i, y_i^*)\}_{i=1}^n$ and $(x_{n+1}, y_{n+1}^*)$ are exchangeable. Given a scoring function $g_{\text{LF}}$ acting on subintervals of the $x_i$, define the conformal score $S_{\text{LF}}(x_i, y_i^*)$ as in \Cref{eq:left-filter-score}. Let $\hat q$ be the $\tfrac{\lceil{(n+1)(1-\alpha)}\rceil}{n}$ quantile of conformal scores $\{S_i\}_{i=1}^n = \{S_{\text{LF}}(x_i, y^*_i)\}_{i=1}^n$. Then prediction sets constructed as $C_{\text{LF}}(x_{n+1}; \hat q) = F_{\text{LF}}(x_{n+1}; \hat q)$ satisfy $1 - \alpha \leq\mathbb{P}[y_{n+1}^* \in C_{\text{LF}}(x_{n+1}; \hat q)] < 1 - \alpha + \frac{1}{n+1}$.
\end{restatable}
\begin{proof}
\looseness=-1 Since the fixed function $S_{\text{LF}}$ is applied to each element of the exchangeable sequence $(x_1, y^*_1), \dots, (x_{n+1}, y^*_{n+1})$, the resulting sequence $S_1, \dots, S_{n+1}$ is also exchangeable. We assume for simplicity that the scores are non-degenerate, because noise can easily be added to break ties. Let $S_{(1)} \le S_{(2)} \le \dots \le S_{(n)}$ be the order statistics of the calibration scores so that the empirical quantile $\hat{q}$ can be defined as $S_{(k)}$, where $k = \lceil (n+1)(1-\alpha) \rceil$ (assuming $n\geq \tfrac{1}{\alpha} - 1$ to ensure $k \leq n$). The event $S_{n+1} \leq \hat{q}$ occurs if and only if $S_{n+1}$ is among the $k$ smallest scores overall. From exchangeability, all orderings are equally likely, so $\mathbb{P}[S_{n+1} \leq \hat{q}] = \frac{k}{n+1}$. By \Cref{thm:LFEquivalence} this event is equivalent to $y^* \in F_{\text{LF}}(x_{n+1}; \hat{q})$. Hence, $\mathbb{P}[y^* \in F_{\text{LF}}(x_{n+1}; \hat q)] = \frac{\lceil (n+1)(1-\alpha) \rceil}{n+1} \ge 1 - \alpha$. 
For the upper bound, we simply note that $\frac{\lceil (n{+}1)(1{-}\alpha) \rceil}{n+1} {<} \frac{1{+}(n{+}1)(1{-}\alpha) }{n+1} {=}  1 {-} \alpha {+} \frac{1}{n+1}$. 
\end{proof}
Beyond generating contiguous prediction sets, LF uses fewer scoring function evaluations on average for inference when $g_\text{LF}$ is monotonic. LF starts with the first suffix $c_{\ell:\ell}$ and evaluates each longer suffix $c_{j:\ell}$ until one with $g_{\text{LF}}(c_{j:\ell}) > q$ is found, then returns $c_{j+1:\ell}$. When errors are evenly distributed over the length $\ell$, the average number of evaluations with a strong $g_\text{LF}$ is only $\tfrac{\ell+1}{2}$.

Of course, there is nothing special about filtering from the left; Right Filtration (RF) can easily be defined which filters from the right, returning a \emph{prefix} of $x$. Using a similarly defined scoring function $g_{\text{RF}}$ we write the relevant prefixes as $\mathcal{I}_{\text{RF}}:= \{c_{1: k}\}_{k=0}^{\ell}$ and define the right-filtering function $F_{\text{RF}}(x; q) := \argmax_{c_{1:k} \in \mathcal{I}_{\text{RF}}} \big(\vert c_{1: k}\vert \mid g_{\text{RF}}(c_{1: k}) \leq q\big)$. Based on these definitions, the conformal score is $S_{\text{RF}}(x, y^*) := \inf\big(q\in \mathbb{R^+} \mid y^* \in F_{\text{RF}}(x; q)\big)$, and prediction sets are generated as $C_{\text{RF}}(x_{n+1}; \hat q) = F_{\text{RF}}(x_{n+1}; \hat q)$. RF gives coverage as a straightforward extension of \Cref{thm:LF} and is advantageous over LF if agents tend to fail earlier in their trajectory rather than later.
\subsubsection{Two-Way Filtration}\label{sec:twf}
\begin{wrapfigure}[9]{r}{0.48\textwidth}
    \centering
    \includegraphics[width=0.95\linewidth, trim={48 4 30 0}, clip]{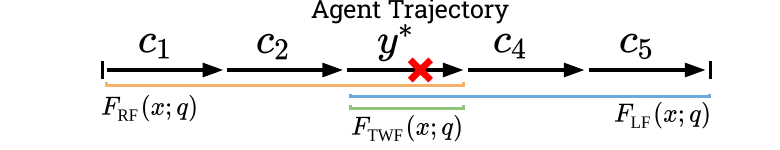}
    \caption{\footnotesize Two-way filtering with $F_\text{TWF}(x;q)$ uses the intersection of filtering from the right and left. The smallest $q$ which retains the decisive error $y^*$  is the conformal score $S_\text{TWF}(x, y^*)$.}
    \label{fig:TWF}
\end{wrapfigure}
Filtering from one direction has downsides; when the decisive error is at the start (end), the suffix (prefix) which covers the ground truth will contain the entire trajectory. Moreover, when the decisive error is near the middle, neither LF nor RF can isolate it. However, these limitations can be avoided by building on LF and RF with Two-Way Filtration (TWF). Bidirectional filtering allows more precise isolation of the decisive error in principle, regardless of where in the trajectory it occurs.

There are many possible ways to define a bidirectional filtration. We present a version that uses the \emph{intersection} of left and right filtered subintervals. Using $g_\text{LF}$ and $g_{\text{RF}}$ as above, we define the two-way filtering function as
\begin{equation}\label{eq:filtering_two_way}
    F_{\text{TWF}}(x; q) := F_{\text{LF}}(x; q) \cap F_{\text{RF}}(x; q).
\end{equation}
This function finds the longest suffix and prefix that each score at most $q$, and takes their intersection. Note that this function still satisfies $F_{\text{TWF}}(x; \infty) = x$, as well as nesting: (All formal proofs are given in \Cref{app:twf-proof})
\begin{restatable}{lemma}{TWFNesting}
\label{thm:TWFNesting}
For any $x$ and thresholds $0\leq q_1 \leq q_2$, $F_{\text{TWF}}(x; q_1) \subseteq F_{\text{TWF}}(x; q_2)
$. 
\end{restatable}
\begin{proofsketch}
The result follows from nesting for L/RF, and set inclusion under intersection: for any sets $A, B, C ,D$, if $A \subseteq B$ and $C \subseteq D$, then $A \cap C \subseteq B \cap D$.
\end{proofsketch}
On top of this we can define the conformal score function
\begin{equation}\label{eq:two-way-filter-score}
    S_{\text{TWF}}(x, y^*) := \inf\big(q\in \mathbb{R^+} \mid y^*\in F_{\text{TWF}}(x; q)\big),
\end{equation}
which operates the same way as before, effectively finding the smallest threshold $q$ under which two-way filtering retains the decisive error. Although $S_{\text{TWF}}$ appears to involve three optimizations, its computation can be simplified to only two scoring function evaluations:

\begin{restatable}{proposition}{TWFComputation}
\label{thm:TWFComputation}
$S_{\text{TWF}}(x, y^*)$ as defined in \Cref{eq:two-way-filter-score} can be expressed as
\begin{equation}\label{eq:TWFIntermediate}
S_{\text{TWF}}(x, y^*) = \max\big(S_{\text{LF}}(x, y^*), S_{\text{RF}}(x, y^*)\big).
\end{equation}
When $g_\text{LF}$ and $g_\text{RF}$ obey the monotonicity condition in \Cref{eq:g_LF_nesting}, $S_{\text{TWF}}$ further simplifies to
\begin{equation}\label{eq:TWFComputation}
    S_{\text{TWF}}(x, y^*) = \max\big(g_\text{LF}(c_{j^*:\ell}), g_\text{RF}(c_{1:j^*})\big).
\end{equation}
\end{restatable}
\begin{proofsketch}
$q$ is a valid threshold ($y^*\in F_{\text{TWF}}(x; q)$) only when $y^*$ is included by both $F_{\text{LF}}$ and $F_{\text{RF}}$, which means $q$ is sufficiently high for both $S_\text{LF}$ and $S_\text{RF}$. We must use the higher of these two thresholds, giving \Cref{eq:TWFIntermediate}.
\Cref{eq:TWFComputation} follows from substituting in \Cref{thm:LFComputation}.
\end{proofsketch}
Like for LF, the TWF conformal algorithm computes $S_{\text{TWF}}$ for each calibration datapoint, sets the conformal threshold $\hat q$ as the $\frac{\lceil(n+1)(1-\alpha)\rceil}{n}$ quantile, and predicts $C_\text{TWF}(x_{n+1};\hat q) = F_{\text{TWF}}(x_{n+1}; \hat q)$ for any test datapoint $x_{n+1}$. $F_{\text{TWF}}(x_{n+1}; \hat q)$ will tend to be a short subinterval when $g_{\text{LF}}$ and $g_{\text{RF}}$ both narrow in on the same steps. This algorithm satisfies \Cref{eq:coverage-guarantee-contiguous} with a theorem and proof similar to \Cref{thm:LF} (see \Cref{thm:TWF}). The core prerequisite is the analogue of \Cref{thm:LFEquivalence}:
\begin{restatable}{lemma}{TWFEquivalence}
\label{thm:TWFEquivalence}
For a fixed $\hat q$, we have $S_{\text{TWF}}(x, y^*) \leq \hat{q} \iff y^* \in F_{\text{TWF}}(x; \hat{q})$.
\end{restatable}
\begin{proofsketch}
From the use of intersection in $F_{\text{TWF}}(x; \hat{q})$ (\Cref{eq:filtering_two_way}), and the result of \Cref{thm:LFEquivalence}, $y^*\in F_{\text{TWF}}(x, \hat q) \iff S_{\text{LF}}(x, y^*)\leq\hat q \wedge S_{\text{RF}}(x, y^*)\leq\hat q$. Equivalently we write $\max\big(S_{\text{LF}}(x, y^*),S_{\text{RF}}(x, y^*)\big) \leq \hat q$, which is the same as $S_{\text{TWF}}(x, y^*)\leq \hat{q}$ (\Cref{thm:TWFComputation}).
\end{proofsketch}
While we described these algorithms using the language of agent trajectories, they can equally be applied to any type of sequential data $x$ with ground truth $y^*\in x$, for example electronic health records where $x$ is a sequence of vitals measurements, and $y^*$ is the first warning sign that should trigger medical intervention \cite{razavian2015temporal}.

\begin{wraptable}[6]{r}{0.5\textwidth}
\centering
\setlength{\tabcolsep}{2.5pt}
\setlength{\extrarowheight}{-1pt}
\setlength{\aboverulesep}{0.5ex}
\setlength{\belowrulesep}{0.5ex}
\setlength{\cmidrulesep}{0.3ex}
\vspace{-14pt}
\caption{Conformal Algorithm Properties}
\label{tab:comparison}
\footnotesize
\begin{tabular}{lccc}
\toprule
Method & Contiguous & Inference NFE & Cov. Upper Bound\\
\midrule
VCP & \textcolor{red}{$\bm{\times}$} & $\ell$ & \textcolor{darkgreen}{$\bm{\checkmark}$} \\
CRSVP &  \textcolor{darkgreen}{$\bm{\checkmark}$} & $\ell$ & \textcolor{red}{$\bm{\times}$}\\
L/RF  & \textcolor{darkgreen}{$\bm{\checkmark}$} & $\approx\frac{1}{2}(\ell+1)$ & \textcolor{darkgreen}{$\bm{\checkmark}$}\\
TWF  &  \textcolor{darkgreen}{$\bm{\checkmark}$} & $\ell$ & \textcolor{darkgreen}{$\bm{\checkmark}$} \\
\bottomrule
\end{tabular}
\end{wraptable}

The five conformal algorithms we compare are summarized in \Cref{tab:comparison}, showing which generate contiguous prediction sets, the number of function evaluations (NFE) of $g$ needed for each test datapoint in the case where $g$ is strong and errors are uniformly distributed, and whether an upper bound on coverage is guaranteed.\vspace{-8pt}
\subsection{Scoring Functions for Agent Error Attribution}\label{sec:scoring function}

Each of the preceding conformal algorithms requires a scoring function $g$ which estimates how likely a step, or set of steps, is to contain the decisive error. Since MAS traces are primarily composed of textual agent interactions, our scoring functions rely on LLM-based components to map agent inputs and outputs to numerical scores. We compare three LLM regimes as outlined in \Cref{subsec:mas-error-attribution}: 
\begin{enumerate}[leftmargin=*, nosep]
    \item \emph{Naive LLM-as-a-judge} is implemented by prompting gpt-4o-mini to estimate error likelihoods with information about the task and step, as in \cite{zhang2025agent};
    \item \emph{Prompt and context engineering} produces step-level likelihood estimates using multiple LLMs with role-specific prompts, inspired by ECHO \cite{banerjee2025hierarchical}. Scores from LLMs with different roles are averaged to get a final likelihood.
    \item A \emph{fine-tuned specialized LLM} is implemented by following the data generation and training paradigm of AEGIS~\cite{kong2025aegis} to fine-tune a Qwen3-1.7B model \cite{yang2025qwen3technicalreport}.
\end{enumerate}
Additional details on prompts, training, and data are given in \Cref{app:c}. For VCP, step-level scores as described above are used directly. However, L/R/TWF and CRSVP require set-level scores. To enable efficient computation in L/R/TWF we design $g$ to obey monotonicity (\Cref{eq:g_LF_nesting}) by aggregating step-level scores. Specifically, we use summation and normalize by the trajectory length $\ell$ to make conformal scores for different datapoints more comparable: $g(c_{j:k}) = \frac{1}{\ell} \sum_{i=j}^{k} \text{LLM}(c_i)$. We experiment with alternative monotonic aggregations, namely Max and LogSumExp, in \Cref{app: additional_res}, and discuss circumstances when they may be preferred. Below, we evaluate the discriminatory power of $g$ independently from CP algorithms by treating it as a multiclass classifier.
\subsection{Conformal Agent Rollbacks}
\label{sec:rollback}
\begin{wrapfigure}[6]{r}{0.48\textwidth}
    \centering
    \vspace{-26pt}
    \includegraphics[width=0.9\linewidth, trim={45 48 33 0}, clip]{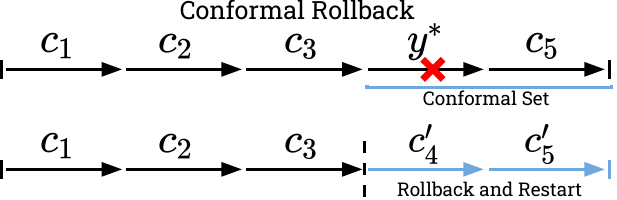}
    \vspace{-6pt}
    \caption{\footnotesize After generating a conformal set for a failed task, we roll back the state of the MAS to the first step in the set, and restart the agent with information on the failed trace.}
    \label{fig:rollback}
\end{wrapfigure}
CP sets for agent error attribution have multiple uses. They can be used by humans for manual debugging; a person can focus only on the steps within the set and have good coverage of true errors. Here, contiguity is a great benefit---it is much easier to debug several consecutive steps than to parse through a scattered set.

However, our main downstream application is automated agent recovery. Knowing that an agent has failed, we wish for it to learn from its mistakes and retry parts of the task. Optimally, the agent's state must be rolled back in the trajectory only far enough to cover the decisive error. Rolling back further is inefficient. However, error attribution accuracy is low for point prediction \cite{zhang2025agent, banerjee2025hierarchical, zhu2025raffles}.

Prediction sets with coverage guarantees provide the solution for automated rollbacks of failed trajectories. Given a conformal set, we roll back the state of the MAS to the first step in the set, and restart the trajectory, as depicted in \Cref{fig:rollback}. The coverage guarantee ensures with high confidence that we roll back far enough to correct the decisive error, while contiguity ensures that we don't roll back excessively far. Upon restarting, we add information to the prompt about the steps taken previously, and instruct the MAS to replan its trajectory to avoid making the same mistakes.

\section{Experiments}
\label{sec:experiments}

\subsection{Datasets}
\label{sec:datasets}
We evaluate CP algorithms, scoring functions, and downstream task performance on both a real-world benchmark and synthetic MAS traces. The \textbf{Who\&When} dataset~\citep{zhang2025agent} (MIT license) is a benchmark for step-level error attribution in MAS. Each of the 184 real-world examples contains a full agent execution trace annotated with the decisive error step.\footnote{We remove inconsistent datapoints where the error index is greater than the trajectory length.} Who\&When is small-scale and provides limited variety over agents and tasks, but it is the only existing academic dataset with step-level error annotations by humans.

\begin{wrapfigure}[12]{r}{0.48\textwidth}
    \centering
    \vspace{-10pt}
    \includegraphics[width=0.99\linewidth]{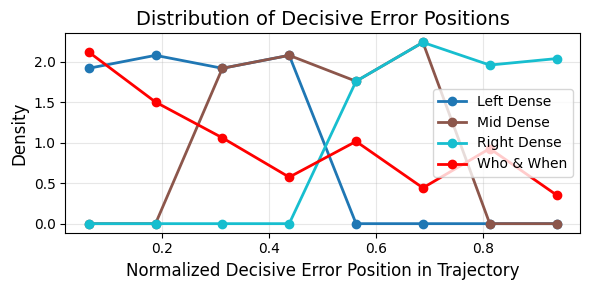}
    \vspace{-16pt}
    \caption{\footnotesize Normalized decisive error position distributions for the real-world benchmark (Who\&When), and controlled error injection datasets (Left/Mid/Right Dense) on GSM8k.
    }
    \label{fig:dense}
\end{wrapfigure}
In the absence of other real-world data, and to provide greater variety, we follow existing practices \citep{kong2025aegis} to synthetically generate failed agent trajectories through error injection. To induce errors at controlled steps, we inject instructions to create errors into the agent's prompt, using the error mode taxonomy of~\citet{pan2025why}. For a given trajectory, one step is selected uniformly at random, and the agent is restarted in that state with a modified prompt describing the desired error mode. Given that the overall trajectory fails, the injected step is labeled as the decisive error. We use two mathematical reasoning datasets, MATH~\citep{hendrycks2021math} and GSM8k~\citep{cobbe2021gsm8k} (MIT licenses), paired with two representative multi-agent architectures, MACNET~\citep{qian2025macnet} and DyLAN~\citep{liu2024dyna}. For each dataset–architecture combination, we generate 1200 failed trajectories.

Complex tasks generally have some steps that are harder to complete than others, leading to non-uniform decisive error distributions. This distribution is shown in \Cref{fig:dense} for the Who\&When dataset, which tends to have decisive errors early on. To mimic this property, and explore the impact of distribution on error attribution performance, we construct variants of the synthetic data by dividing the normalized trajectory length into thirds, and subsampling conditional on decisive error location. These are the \textbf{Left}, \textbf{Mid}, and \textbf{Right Dense} variants of GSM8k in \Cref{fig:dense}. Additional details and analysis are provided in Appendix~\ref{app:c}.
\subsection{Evaluation Metrics}
Scoring functions $g$ can be viewed as classifiers on an $\ell$-way task---predicting the decisive error location. Hence, we evaluate them using AUROC, AUPRC, and accuracy. Since $\ell$ can be large, baseline levels of these metrics are very low.

\looseness=-1 Given a test dataset $\{(x_i, y_i^*)\}_{i=n+1}^{n+m}$ and predicted sets $\{C(x_i)\}_{i=n+1}^{n+m}$, we evaluate conformal agent error attribution with the following metrics: 
\begin{equation}\label{eq:eval-metrics}
\text{EC} \;=\; \frac{1}{m}\sum_{i=n+1}^{n+m} \mathbb{I}\!\left[y_i^* \in C(x_i)\right]; \qquad \text{RR} \;=\; \frac{1}{m}\sum_{i=n+1}^{n+m}
\left(1 - \frac{|C(x_i)|}{\ell_i}\right).
\end{equation}
\textbf{Empirical Coverage} (EC) measures how often the predicted set contains the decisive error. EC should be at least $1-\alpha$, the target coverage, and respect the corresponding upper bound where applicable. \textbf{Removal Rate} (RR) measures how much of the trajectory is filtered out, essentially measuring prediction set size, where $\ell_i$ denotes the number of steps in $x_i$. A higher removal rate indicates smaller conformal sets, and more precise localization of the error step, given equal EC.

Finally, rollback performance is measured across three axes: Success Rate - the fraction of tasks successfully completed after the rollback; Coverage - the fraction of tasks where the state is rolled back at least to the decisive error; and Cost - the fraction of steps rolled back, needing to be redone. As a baseline, we test the agent restarting from the single most likely predicted step (Top-1).

\section{Results}
\label{sec:results}

\begin{figure}[t]
    \centering
    \begin{minipage}[t]{0.245\textwidth}
        \centering
        \includegraphics[width=\linewidth]{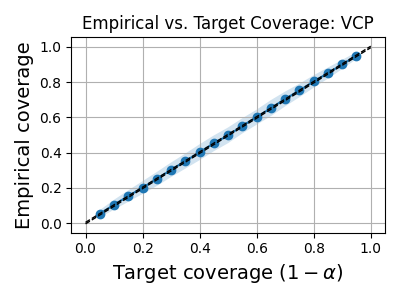}
    \end{minipage}
    \hfill
    \begin{minipage}[t]{0.245\textwidth}
        \centering
        \includegraphics[width=\linewidth]{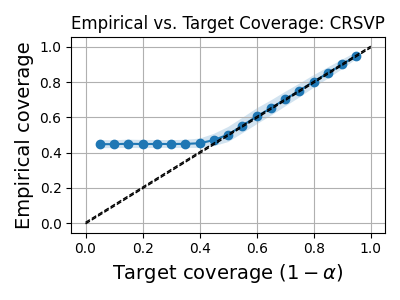}
    \end{minipage}
    \hfill
    \begin{minipage}[t]{0.245\textwidth}
        \centering
        \includegraphics[width=\linewidth]{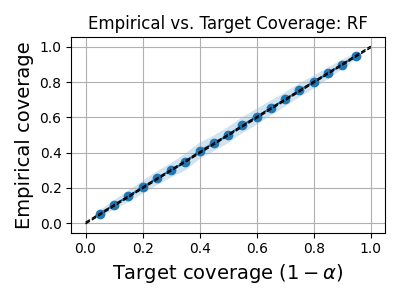}
    \end{minipage}
    \hfill
    \begin{minipage}[t]{0.245\textwidth}
        \centering
        \includegraphics[width=\linewidth]{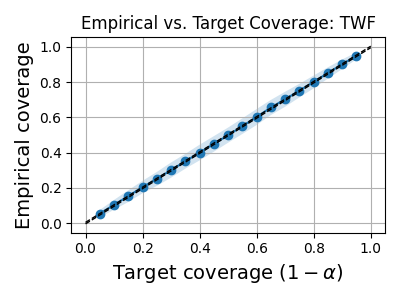}
    \end{minipage}
    \caption{\footnotesize Empirical Coverage vs.\ Target Coverage for CP methods with the fine-tuned LLM scoring function on MATH.}
    \label{fig:coverage}
\end{figure}
\subsection{Empirical Verification of Coverage Guarantee}
First, we empirically verify that CP algorithms considered in this work satisfy their respective coverage guarantees from \Cref{sec:method}. \Cref{fig:coverage} shows empirical coverage as a function of the target coverage $1-\alpha$, where we average EC over 1000 random, even splits of the calibration and test sets. The shaded regions show one standard deviation. EC is above the lower bound for all methods, consistent with the theoretical guarantees. VCP, RF, and TWF all obey their respective upper bounds, leading to exact linear behaviour. A deviation is observed in the low $1-\alpha$ regime for CRSVP, but this is not a violation as CRSVP does not provide an upper bound on coverage.
\subsection{Scoring Function Evaluation}
\begin{wraptable}[10]{r}{0.5\textwidth}
\centering
\setlength{\tabcolsep}{2.5pt}
\setlength{\extrarowheight}{-1pt}
\setlength{\aboverulesep}{0.5ex}
\setlength{\belowrulesep}{0.5ex}
\setlength{\cmidrulesep}{0.3ex}
\caption{\footnotesize Step-level discrimination metrics across scoring functions on combined MATH and GSM8k test set.}
\footnotesize
\begin{tabular}{lccc}
\toprule
 Scoring function & AUROC & AUPRC & Accuracy \\
\midrule
 Random Guessing (baseline) & 0.500 & 0.118 & 0.118 \\
 \midrule
Naive LLM (gpt-4o-mini)
&0.519  &0.128 &0.110  \\
Role Prompted (gpt-4o-mini)
&0.554  &0.161  &0.164 \\
Fine-tuned (Qwen3-1.7B)
 & \textbf{0.762} & \textbf{0.382} & \textbf{0.731} \\
\bottomrule
\end{tabular}
\label{tab:scoring_function}
\end{wraptable}
Next, we compare the three regimes of scoring functions for discriminatory power as binary classifiers. \cref{tab:scoring_function} shows results on a combination of GSM8k and MATH test sets, matching the training distribution for the fine-tuned LLM. Increasing the complexity of $g$ does pay off; fine-tuning on the error classification task considerably increases its power. The naive prompted LLM is hardly better than random guessing, while prompt engineering provides small improvements. These results are consistent with performance levels observed in prior studies~\cite{zhang2025agent, banerjee2025hierarchical, kong2025aegis}.

\begin{table*}[t]
\centering
\scriptsize
\setlength{\tabcolsep}{2.5pt}
\setlength{\extrarowheight}{-1pt}
\setlength{\aboverulesep}{0.5ex}
\setlength{\belowrulesep}{0.5ex}
\setlength{\cmidrulesep}{0.3ex}
\caption{\footnotesize Removal rate ($\pm$ std over 1000 calibration/test splits) for conformal algorithms and scoring functions at $1-\alpha=0.8$. Higher removal rates indicate more precise error attribution, with the best performing conformal algorithm in each column bolded among algorithms which produce contiguous sets. Averages are taken over conformal algorithms to evaluate scoring functions.  \textcolor{red}{$\bm{\times}$} indicates that fine-tuning is not possible due to insufficient data.}

\resizebox{1.0\textwidth}{!}{%
\begin{tabular}{clccccccccc}
\toprule
\multirow{0.5}{*}{}
&\multirow{2.5}{*}{\makecell{Conformal \\ Method}}
& \multirow{2.5}{*}{\makecell{Who\&When}}
& \multicolumn{2}{c}{GSM8k Left Dense}
& \multicolumn{2}{c}{GSM8k Mid Dense}
& \multicolumn{2}{c}{GSM8k Right Dense} 
& \multicolumn{2}{c}{MATH} \\
\cmidrule(lr){4-5} \cmidrule(lr){6-7} \cmidrule(lr){8-9} \cmidrule(lr){10-11}
&
&
& DyLAN & MACNET
& DyLAN & MACNET
& DyLAN & MACNET
& DyLAN & MACNET \\
\midrule

\multirow{5}{*}{\rotatebox[origin=c]{90}{\parbox{1.75cm}{\centering Naive LLM \\ gpt-4o-mini}}}&
Vanilla
& 0.20$\pm$0.06
& 0.21$\pm$0.02 & 0.24$\pm$0.03
& 0.18$\pm$0.03 & 0.15$\pm$0.03
& 0.16$\pm$0.03 & 0.17$\pm$0.03
& 0.19$\pm$0.02 & 0.20$\pm$0.03 \\
\cmidrule(lr){2-11} 
& CRSVP
& 0.21$\pm$0.04
& 0.58$\pm$0.03 & 0.54$\pm$0.02
& 0.29$\pm$0.04 & 0.21$\pm$0.04
& 0.21$\pm$0.03 & 0.10$\pm$0.03
& 0.30$\pm$0.03 & 0.20$\pm$0.03 \\
& Left Filtering
& 0.12$\pm$0.04
& 0.10$\pm$0.01 & 0.24$\pm$0.01
& 0.31$\pm$0.01 & 0.47$\pm$0.02
& 0.53$\pm$0.01 & \textbf{0.60$\pm$0.02}
& 0.18$\pm$0.02 & 0.19$\pm$0.02 \\
& Right Filtering
& \textbf{0.31$\pm$0.04}
& \textbf{0.64$\pm$0.02} & 0.71$\pm$0.01
& 0.43$\pm$0.02 & 0.45 $\pm$0.01
& 0.20$\pm$0.02 & 0.20$\pm$0.01
& \textbf{0.27$\pm$0.02} & 0.20$\pm$0.02 \\
& Two-Way
& 0.19$\pm$0.04
& 0.17$\pm$0.03 & 0.22  $\pm$0.03
& \textbf{0.59$\pm$0.03} & \textbf{0.59$\pm$0.01}
& 0.40$\pm$0.03 & 0.21$\pm$0.03 
& \textbf{0.27$\pm$0.02} & 0.15$\pm$0.03 \\
\cmidrule(lr){2-11} 
& Average
& 0.21
& 0.34 & 0.39
& 0.36 & 0.37
& 0.30 & 0.26
& 0.24 & 0.19 \\
\midrule
\multirow{5}{*}{\rotatebox[origin=c]{90}{\parbox{1.5cm}{\centering Role-prompted \\ gpt-4o-mini}}}&
 Vanilla
& 0.19$\pm$0.06
& 0.30$\pm$0.04 & 0.24$\pm$0.04
& 0.27$\pm$0.04 & 0.27$\pm$0.03
& 0.22$\pm$0.04 & 0.25$\pm$0.04
& 0.19$\pm$0.03 & 0.24$\pm$0.04 \\
\cmidrule(lr){2-11} 
& CRSVP
& 0.23$\pm$0.04
& 0.27$\pm$0.04 & 0.21$\pm$0.05
& 0.27$\pm$0.04 & 0.26$\pm$0.05
& 0.27$\pm$0.04 & 0.22$\pm$0.05
& \textbf{0.27$\pm$0.03} & 0.22$\pm$0.03 \\
& Left Filtering
& 0.12$\pm$0.04
& 0.10$\pm$0.02 & 0.31$\pm$0.01
& 0.36$\pm$0.01 & 0.53$\pm$0.01
& \textbf{0.64$\pm$0.02} & \textbf{0.60$\pm$0.02}
& 0.19$\pm$0.02 & 0.19$\pm$0.02 \\
& Right Filtering
& \textbf{0.30$\pm$0.05}
& \textbf{0.63$\pm$0.02} & 0.59$\pm$0.02
& 0.43$\pm$0.01 & 0.33$\pm$0.01
& 0.21$\pm$0.01 & 0.10$\pm$0.04
& \textbf{0.28$\pm$0.02} & 0.19$\pm$0.02 \\
& Two-Way
& 0.22$\pm$0.02
& 0.18$\pm$0.02 & 0.21$\pm$0.02
& \textbf{0.62$\pm$0.02} & \textbf{0.59$\pm$0.02}
& 0.42$\pm$0.02 & 0.20$\pm$0.02
& \textbf{0.29$\pm$0.02} & 0.18$\pm$0.02 \\
\cmidrule(lr){2-11} 
& Average
& 0.21
& 0.30 & 0.31
& 0.39 & 0.40
& 0.35 & 0.27
& 0.24 & 0.20 \\

\midrule
\multirow{5}{*}{\rotatebox[origin=c]{90}{\parbox{1.5cm}{\centering Fine-tuned\\Qwen3-1.7B}}}&
 Vanilla
& \textcolor{red}{$\bm{\times}$} 
& 0.78$\pm$0.03 & 0.82$\pm$0.02
& 0.66$\pm$0.03 & 0.75$\pm$0.02
& 0.34$\pm$0.04 & 0.63$\pm$0.03
& 0.33$\pm$0.03 & 0.62$\pm$0.03 \\
\cmidrule(lr){2-11} 
& CRSVP
& \textcolor{red}{$\bm{\times}$}
& \textbf{0.69$\pm$0.06} & \textbf{0.80$\pm$0.02} 
& 0.36$\pm$0.05 & 0.38$\pm$0.07
& 0.24$\pm$0.03 & 0.29$\pm$0.06
& \textbf{0.32$\pm$0.03} & \textbf{0.30$\pm$0.04} \\
& Left Filtering
& \textcolor{red}{$\bm{\times}$}
& 0.09$\pm$0.01 & 0.10$\pm$0.02
& 0.31$\pm$0.02 & 0.35$\pm$0.01
& 0.52$\pm$0.01 & \textbf{0.60$\pm$0.01}
& 0.18$\pm$0.02 & 0.19$\pm$0.02 \\
& Right Filtering
& \textcolor{red}{$\bm{\times}$}
& \textbf{0.65$\pm$0.02} & 0.63$\pm$0.02
& 0.44$\pm$0.01 & 0.33$\pm$0.01
& 0.20$\pm$0.01 & 0.10$\pm$0.02
& 0.27$\pm$0.01 & 0.20$\pm$0.02 \\
& Two-Way
& \textcolor{red}{$\bm{\times}$}
& 0.18$\pm$0.02 & 0.20$\pm$0.05
& \textbf{0.63$\pm$0.02} & \textbf{0.60$\pm$0.02}
& 0.40$\pm$0.02 & 0.19$\pm$0.05
& \textbf{0.29$\pm$0.02} & 0.19$\pm$0.02 \\
\cmidrule(lr){2-11} 
& Average
& \textcolor{red}{$\bm{\times}$}
& 0.48 & 0.51
& 0.48 & 0.48
& 0.33 & 0.36 
& 0.28 & 0.30 \\
\bottomrule
\end{tabular}
}
\label{tab:removal_rate_comparison_80}
\end{table*}\vspace{-2pt}
\subsection{Conformal Agent Error Attribution Evaluation}
\vspace{-4pt}
Our main comparison of error attribution examines the removal rate (\Cref{eq:eval-metrics}) for conformal algorithms and scoring functions across real and synthetic datasets, including trajectories created through two agentic frameworks. \cref{tab:removal_rate_comparison_80} displays the results using a target coverage of $80\%$. Results are shown as the mean and standard deviation over 1000 random splits of the calibration and test data.

Our first observation is that the efficacy of conformal algorithms depends on the distribution of errors in the dataset. The real-world Who\&When dataset has decisive errors heavily skewed toward early steps in the trajectory (see \Cref{fig:dense}). This distribution naturally aligns with RF which can return a short prefix, and accordingly shows the strongest performance. The synthetic datasets demonstrate this behaviour even more clearly. LF is the strongest algorithm on right-dense data, and TWF successfully isolates errors in the middle of trajectories. Conformal error attribution is most successful when the CP algorithm is matched to the data's error distribution. Luckily, this is simple to achieve in practice. Applying CP already requires a calibration dataset from a distribution similar to the test set. Practitioners can visualize the error distribution of their calibration data like in \Cref{fig:dense}, and choose the appropriate filtering algorithm, or automate the selection  based on the mean error location of trajectories, for example.

Notably, filtering algorithms are less dependent on the discriminatory power of the scoring function than VCP and CRSVP. While the baseline approaches tend to require the fine-tuned LLM scoring function to achieve high RR, filtering methods are largely insensitive. This means that L/R/TWF can be used with simple, easier to design scoring functions as long as they are properly matched to the data's error distribution. This is especially important for real-world data where labeling is expensive; Who\&When only supplies 184 labeled datapoints---enough to calibrate and test, but not nearly enough to also fine-tune an LLM. Our filtering algorithms remain more practical in this setting.
\subsection{Computational Efficiency}\label{sec:efficiency}
\begin{wraptable}[9]{r}{0.5\textwidth}
\centering
\setlength{\tabcolsep}{2.5pt}
\setlength{\extrarowheight}{-1pt}
\setlength{\aboverulesep}{0.5ex}
\setlength{\belowrulesep}{0.5ex}
\setlength{\cmidrulesep}{0.3ex}
\vspace{-14pt}
\caption{\footnotesize Inference cost comparison for conformal methods in terms of average NFE of $g$ (Fine-tuned LLM) per trajectory with 80\% target coverage on GSM8k variants.}
\small
\resizebox{0.5\textwidth}{!}{%
\begin{tabular}{lccc}
\toprule
 Conformal Method & Left Dense & Mid Dense & Right Dense \\
\midrule
Vanilla
&8.50 & 8.50  &8.50  \\
CRSVP
&8.50  &8.50  &8.50  \\
Left Filtering
&7.50  &5.64  &\textbf{3.70} \\
 Right Filtering
&\textbf{3.08}  &\textbf{5.16}  &7.09 \\
Two-Way
&8.50  &8.50  &8.50 \\
\bottomrule
\end{tabular}
}
\label{tab:time_analysis}
\end{wraptable}

As noted throughout \Cref{sec:method} and \Cref{tab:comparison}, at inference existing conformal methods including VCP and CRSVP require evaluating the scoring function $g$ on each step in the trajectory. LF and RF need only evaluate steps from one direction until their threshold is crossed, which can be much more efficient. To demonstrate this quantitatively, we count the average NFE of $g$ used by each algorithm when predicting on the GSM8k test set variants. The result in \Cref{tab:time_analysis} confirms that VCP, CRSVP, and TWF use exactly $\ell$ function calls per trajectory. In contrast, LF shows a clear advantage on right-dense error distributions, while RF uses a mere 36\% of the calls on left-dense data where most errors occur early. This result emphasizes the cost advantages of selecting an algorithm that has good implicit bias for the distribution observed in calibration data.
\subsection{Conformal Agent Rollbacks}
\begin{wraptable}[14]{r}{0.52\textwidth}
\centering
\setlength{\tabcolsep}{3pt}
\setlength{\extrarowheight}{-1pt}
\setlength{\aboverulesep}{0.5ex}
\setlength{\belowrulesep}{0.5ex}
\setlength{\cmidrulesep}{0.3ex}
\vspace{-10pt}
\caption{\footnotesize Automated rollback results using CP sets from VCP and LF compared to Top-1 predictions. Datasets are variants of GSM8k, and the scoring function uses the fine-tuned LLM.}
\footnotesize
\begin{tabular}{llccc}
\toprule
Dataset & Method & Success Rate $\uparrow$ & Coverage & Cost $\downarrow$ \\
\midrule
\multirow{2}{*}{Left Dense}
& Top-1 & 0.76\tiny$\pm$0.05\scriptsize & 0.92\tiny$\pm$0.03\scriptsize & 0.83\tiny$\pm$0.02\scriptsize \\
& VCP & 0.73\tiny$\pm$0.05\scriptsize & 0.85\tiny$\pm$0.03\scriptsize & 0.79\tiny$\pm$0.02\scriptsize \\
& LF    & 0.77\tiny$\pm$0.04\scriptsize & 0.81\tiny$\pm$0.05\scriptsize & 0.90\tiny$\pm$0.00\scriptsize \\
\midrule
\multirow{2}{*}{Mid Dense}
& Top-1 & 0.67\tiny$\pm$0.04\scriptsize & 0.86\tiny$\pm$0.04\scriptsize & 0.56\tiny$\pm$0.02\scriptsize \\
& VCP & 0.59\tiny$\pm$0.05\scriptsize & 0.95\tiny$\pm$0.02\scriptsize & 0.65\tiny$\pm$0.01\scriptsize \\
& LF    & 0.68\tiny$\pm$0.05\scriptsize & 0.81\tiny$\pm$0.03\scriptsize & 0.65\tiny$\pm$0.01\scriptsize \\
\midrule
\multirow{2}{*}{Right Dense}
& Top-1 & 0.71\tiny$\pm$0.05\scriptsize & 0.86\tiny$\pm$0.04\scriptsize & 0.50\tiny$\pm$0.02\scriptsize \\
& VCP & 0.70\tiny$\pm$0.05\scriptsize & 1.00\tiny$\pm$0.00\scriptsize & 0.66\tiny$\pm$0.02\scriptsize \\
& LF    & 0.75\tiny$\pm$0.04\scriptsize & 0.82\tiny$\pm$0.04\scriptsize & 0.40\tiny$\pm$0.01\scriptsize \\
\bottomrule
\end{tabular}
\label{tab:rollback}
\end{wraptable}

As noted in \cref{sec:rollback}, CP sets can enable automated rollbacks for failed tasks with high confidence that decisive errors are captured. We evaluate automated rollbacks on the GSM8K test set variants. For each dataset, we randomly sample 100 MACNET trajectories and score them with the fine-tuned LLM scoring function. As a baseline method, we use the single step in the trajectory with the highest likelihood (Top-1) as the restart point. For conformal rollbacks we predict a conformal set using either the VCP or LF algorithm with a target coverage of 80\%, and then restart at the first step in the set. In all cases we restart the MACNET agent with a modified prompt containing the previously failed trace to help the system avoid repeating the same mistakes. As shown in \cref{tab:rollback}, LF achieves the highest success rate by a small margin, within one standard deviation of baselines. However, LF is the only method which gives control over coverage: Top-1 gives no guarantee, while VCP overcovers in this setting because its sets are not contiguous.

\section{Conclusion}
\label{sec:conclusion}
\vspace{-4pt}
In this work we designed novel set-prediction algorithms for the agent error attribution task. By taking account of the data's sequential nature, our filtering-based conformal algorithms showed strong ability to precisely isolate decisive errors in agent trajectories within contiguous sets, and with much less compute needed at inference time. Our conformal algorithms can be used as one part of an agent improvement pipeline, allowing humans to more efficiently understand and debug errors in their agents, or enabling fully automated rollbacks wherein agents correct their own mistakes. The limitations of our methods include: the requirement of exchangeable data, which could be loosened with standard CP methods to handle distribution shifts \cite{gibbs2021adaptive}; the need to match algorithm choice to data distribution, which can easily be done by checking the calibration dataset; and the assumption that failed trajectories have a single decisive error, which can be naturally expanded to the setting of an arbitrary number of agent errors by using a CP guarantee on the recall level over \emph{all} errors \cite{kuwahara2025document}.


\bibliographystyle{plainnat}
\bibliography{bib}
\clearpage


\appendix

\section{Theorems and Proofs}
\label{app:a}

In this appendix we provide complete proofs of the theorems stated in the main text.

\subsection{Vanilla Conformal Prediction}\label{app:vcp-proof}

The theorem showing valid coverage for VCP with variable number of classes is essentially the same as the standard result for conformal prediction, for instance as presented by \citep{angelopoulos2022gentle}, since the number of classes per datapoint does not affect exchangeability of the conformal scores. Although not novel, we present the theorem below for completeness. Note that following tradition in CP, the scores in this algorithm reflect \emph{non-conformity} in that they take higher values when $y$ disagrees with $x$.

\begin{restatable}{theorem}{VCPTheorem}
\label{thm:vcp}
Suppose $\{(x_i, y_i^*)\}_{i=1}^n$ and $(x_{n+1}, y_{n+1}^*)$ are exchangeable, where the number of possible classes $\ell_i$ varies per datapoint. Define $\hat q$ as the $\tfrac{\lceil{(n+1)(1-\alpha)}\rceil}{n}$ quantile of the scores $\{S_i\}_{i=1}^n = \{S(x_i, y_i^*)\}_{i=1}^n$. Then prediction sets constructed as $C(x_{n+1}; \hat q)=\{y\in \mathcal{Y}(x_{n+1}) \mid S(x_{n+1}, y) \leq \hat q\}$ satisfy $\mathbb{P}[y_{n+1}^* \in C(x_{n+1};\hat q)] \geq 1 - \alpha$.
\end{restatable}
\begin{proof}
Since the fixed scoring function $S$ is applied to each element of the exchangeable sequence $(x_1, y^*_1), \dots, (x_{n+1}, y^*_{n+1})$, the resulting sequence of scalar scores $S_1, \dots, S_{n+1}$ is also exchangeable. We assume for simplicity that the scores are non-degenerate, but ties can be handled by defining $S$ to add a small amount of noise to its output.  
Let $S_{(1)} \le S_{(2)} \le \dots \le S_{(n)}$ be the order statistics of the calibration scores so that the empirical quantile $\hat{q}$ can be defined as $S_{(k)}$, where $k = \lceil (n+1)(1-\alpha) \rceil$ (assuming $n\geq \tfrac{1}{\alpha} - 1$ to ensure $k \leq n$). 
Given the definition of $C(x_{n+1}; \hat q)$, the event $y_{n+1}^* \in C(x_{n+1}; \hat q)$ is equivalent to the event $S_{n+1} \leq \hat{q}$, which occurs if and only if $S_{n+1}$ is among the $k$ smallest scores overall. From exchangeability, all orderings are equally likely, so $\mathbb{P}[S_{n+1} \leq \hat{q}] = \frac{k}{n+1}$. Altogether we have 
\begin{equation}
    \mathbb{P}[y_{n+1}^* \in C(x_{n+1}; \hat{q})] = \frac{k}{n+1} = \frac{\lceil (n+1)(1-\alpha) \rceil}{n+1} \ge \frac{(n+1)(1-\alpha)}{n+1} = 1 - \alpha.
\end{equation}
\end{proof}
We note that the variable number of classes $\ell_i$ per datapoint does not come up in the proof, and is only present in the assumption that such datapoints can still be exchangeable. This is certainly possible when we treat $\ell_i$ as a (random) property of $x$. If desired, one could treat $\ell_i$ as an explicit random variable and draw points as $(x_i, \ell_i, y^*_i)\sim \mathbb{P}$. Another way of viewing this extension would be to set a fixed $\ell$ as some upper bound of possible sizes, such as $\ell=\max \{\ell_i\}$. Then, ordinary conformal prediction can be performed with a model $f$ defined over $\ell$ classes but giving $f_{j}=0$ when $j > \ell_i$, as long as we assume $\ell_{n+1}\leq \ell$.

\subsection{Left (Right) Filtration}\label{app:lf-proof}
Here we provide full formal proofs for each of the steps in \Cref{sec:LF}. For ease of reference, we repeat key definitions and assumptions. A trajectory $x$ is a sequence of a variable number $\ell$ steps $(c_1, \dots, c_\ell)$, and the decisive error $y^*=c_{j^*}$ is the step at index $j^*$. The scoring function $g_\text{LF}$ scores subintervals $c_{j:k}\subseteq x$ with the properties $g_\text{LF}(\varnothing) = 0$ and $g_\text{LF}(x)=\infty$. The filtering function is defined as 
$F_{\text{LF}}(x; q) := \argmax_{c_{j: \ell} \in \mathcal{I}_{\text{LF}}} \big(\vert c_{j: \ell}\vert \mid g_{\text{LF}}(c_{j: \ell}) \leq q\big)$ which optimizes over suffixes $\mathcal{I}_{\text{LF}}:= \{c_{j:  \ell}\}_{j=1}^{\ell+1}$ and satisfies $F_{\text{LF}}(x; \infty) = x$. Meanwhile, the conformal score function is defined as $S_{\text{LF}}(x, y^*) := \inf\big(q\in \mathbb{R^+} \mid y^* \in F_{\text{LF}}(x; q)\big)$. Both of these functions optimize over a set of steps, and it will be convenient to refer to those sets as follows:
\begin{align}
    \mathcal{V}_q &:= \{c_{j:\ell} \in \mathcal{I}_{\text{LF}} \mid g_{\text{LF}}(c_{ j:\ell}) \leq q\},\\
    \mathcal{Q} &:= \{q\in \mathbb{R^+} \mid y^* \in F_{\text{LF}}(x; q) \}.
\end{align}

\LFNesting*
\begin{proof}
$\mathcal{V}_q$ is the set of valid suffixes for $F_{\text{LF}}(x; q)$. For every $c_{ j:\ell} \in \mathcal{V}_{q_1}$ we have $g_{\text{LF}}(c_{j:\ell}) \leq q_1 \leq q_2$, and so $c_{j:\ell} \in \mathcal{V}_{q_2}$. Hence, $\mathcal{V}_{q_1} \subseteq \mathcal{V}_{q_2}$. $F_{\text{LF}}(x; q)$ returns the longest suffix in $\mathcal{V}_{q}$, so we also have $\vert F_{\text{LF}}(x; q_1) \vert \leq \vert F_{\text{LF}}(x; q_2) \vert$. Since the suffixes $c_{j:\ell} \in \mathcal{I}_{\text{LF}}$ are nested (i.e. $j_2 < j_1 \iff c_{j_1:\ell} \subset c_{j_2:\ell}$), we also have $\vert c_{j_1:\ell} \vert \leq \vert c_{j_2:\ell}\vert \implies c_{j_1:\ell} \subseteq c_{j_2:\ell}$, which gives the desired nesting property of $F_{\text{LF}}(x; q)$.
\end{proof}

\LFComputation*
\begin{proof}
$\mathcal{Q}$ is the set of valid thresholds for $S_{\text{LF}}(x, y^*)$. Let $q^* = g_\text{LF}(c_{j^*:\ell})$, and note that $c_{j^*:\ell} \in \mathcal{V}_{q^*}$. Since $\vert c_{j^*:\ell}\vert < \vert c_{j:\ell}\vert$ only when $j<j^*$, and $y^*\in c_{j:\ell}$ for these cases, the longest suffix in $\mathcal{V}_{q^*}$ contains $y^*$. Hence, we have that $y^* \in F_{\text{LF}}(x; q^*)$, and $q^* \in \mathcal{Q}$.

Now consider any $q'<q^*$. Notice that $c_{j^*:\ell} \not\in \mathcal{V}_{q'}$. By monotonicity of $g_\text{LF}$, for all $j< j^*$, $g_\text{LF}(c_{j:\ell}) \geq q^* > q'$. Therefore, the longest suffix in $\mathcal{V}_{q'}$ is shorter than $c_{j^*:\ell}$, and does not contain $y^*$. Hence, $q' \not \in \mathcal{Q}$.

Since $q^* \in \mathcal{Q}$, but all $q' < q^*$ are not, $q^*$ is the minimum of $\mathcal{Q}$, and $S_{\text{LF}}(x, y^*) = q^* = g_\text{LF}(c_{j^*:\ell})$.
\end{proof}

\LFEquivalence*
\begin{proof}
To begin, we show that the infimum in $S_{\text{LF}}(x, y^*)$ is always achieved. Note that the set of suffix scores $g_{\text{LF}}(c_{j:\ell})$ is finite. Considering only the valid suffixes that contain $y^*$, let $q_{\min} = \min \{g_{\text{LF}}(c_{j:\ell}) \mid j \le j^*\}$. Due to the nesting of suffixes, the condition $y^* \in F_{\text{LF}}(x; q)$ is satisfied if and only if there exists a valid suffix with score $\leq q$, which occurs precisely when $ q_{\min} \leq q$. Consequently, $\mathcal{Q}$ is the closed interval $[q_{\min}, \infty)$, and the infimum in $S_{\text{LF}}(x, y^*)$ is achieved at $q_\text{min}$.

Hence we have a chain of equivalences:
\begin{equation}
    S_{\text{LF}}(x, y^*) \leq \hat{q} \iff q_\text{min} \leq \hat q \iff \hat q\in \mathcal{Q} \iff y^* \in F_{\text{LF}}(x; \hat q).
\end{equation}

\end{proof}

\subsection{Two-Way Filtration}\label{app:twf-proof}

Next we cover the proofs of statements in \Cref{sec:twf}. Continuing from the definitions and results above, including their counterparts for RF, for TWF we had $F_{\text{TWF}}(x; q) := F_{\text{LF}}(x; q) \cap F_{\text{RF}}(x; q)$ (where  $F_{\text{TWF}}(x; 0) = \varnothing$ and $F_{\text{TWF}}(x; \infty) = x$), and $S_{\text{TWF}}(x, y^*) := \inf\big(q\in \mathbb{R^+} \mid y^*\in F_{\text{TWF}}(x; q)\big)$.

\TWFNesting*
\begin{proof}
From \Cref{thm:LFNesting} we have $F_{\text{LF}}(x; q_1) \subseteq F_{\text{LF}}(x; q_2)$, and the equivalent statement for $F_{\text{RF}}$ holds. The intersection operation preserves set inclusion: for any sets $A, B, C ,D$, if $A \subseteq B$ and $C \subseteq D$, then $A \cap C \subseteq B \cap D$. Hence, 
\begin{equation}
    F_{\text{TWF}}(x; q_1) = F_{\text{LF}}(x; q_1) \cap F_{\text{RF}}(x; q_1) \subseteq F_{\text{LF}}(x; q_2) \cap F_{\text{RF}}(x; q_2) = F_{\text{TWF}}(x; q_2).
\end{equation}
\end{proof}

\TWFComputation*
\begin{proof}
By the definition of $F_{\text{TWF}}$, $y^*\in F_{\text{TWF}}(x; q)$ if and only if $y^* \in F_{\text{LF}}(x; q)$ and $y^* \in F_{\text{RF}}(x; q)$. The condition $y^* \in F_{\text{LF}}(x; q)$ holds if and only if $q \geq S_{\text{LF}}(x, y^*)$ by the definition of $F_{\text{LF}}(x; q)$ and its nesting property from \Cref{thm:LFNesting}. Similarly, $y^* \in F_{\text{RF}}(x; q)$ if and only if $q \geq S_{\text{RF}}(x, y^*)$. Hence, $y^*\in F_{\text{TWF}}(x; q)$ if and only if $q \geq \max(S_{\text{LF}}(x;  y^*), S_{\text{RF}}(x;  y^*))$. By taking the infimum for $S_{\text{TWF}}(x;  y^*)$, we reach \Cref{eq:TWFIntermediate}.

From \Cref{thm:LFComputation}, when $g_\text{LF}$ obeys its monotonicity condition, we have $S_{\text{LF}}(x, y^*)=g_\text{LF}(c_{j^*:\ell})$, and by a similar argument $S_{\text{RF}}(x;  y^*) = g_{\text{RF}}(c_{1:j^*})$. Substituting these results into \Cref{eq:TWFIntermediate} gives \Cref{eq:TWFComputation}.
\end{proof}

\TWFEquivalence*
\begin{proof}
From the definition of $F_{\text{TWF}}(x, q)$ as an intersection, the event $y^*\in F_{\text{TWF}}(x, \hat q)$ occurs if and only if both $y^*\in F_{\text{LF}}(x, \hat q)$ and $y^*\in F_{\text{RF}}(x, \hat q)$ occur. We have already shown that $y^*\in F_{\text{LF}}(x, \hat q)\iff S_{\text{LF}}(x, y^*)\leq\hat q$ (\Cref{thm:LFEquivalence}), and similarly for RF. Hence, 
\begin{equation}
    y^*\in F_{\text{TWF}}(x, \hat q) \iff S_{\text{LF}}(x, y^*)\leq\hat q \wedge S_{\text{RF}}(x, y^*)\leq\hat q.
\end{equation}
Equivalently we write $\max\big(S_{\text{LF}}(x, y^*),S_{\text{RF}}(x, y^*)\big) \leq \hat q$, which is the same as $S_{\text{TWF}}(x, y^*)\leq \hat{q}$ (\Cref{thm:TWFComputation}).
\end{proof}

\begin{restatable}{theorem}{TWFTheorem}
\label{thm:TWF}
Suppose $\{(x_i, y_i^*)\}_{i=1}^n$ and $(x_{n+1}, y_{n+1}^*)$ are exchangeable. Given scoring functions $g_{\text{LF}}$ and $g_{\text{RF}}$ as defined above, define the conformal score $S_{\text{TWF}}(x_i, y_i^*)$ as in \Cref{eq:two-way-filter-score}. Let $\hat q$ be the $\tfrac{\lceil{(n+1)(1-\alpha)}\rceil}{n}$ quantile of conformal scores $\{S_i\}_{i=1}^n = \{S_{\text{TWF}}(x_i, y^*_i)\}_{i=1}^n$. Then prediction sets constructed as $C_{\text{TWF}}(x_{n+1}; \hat q) = F_{\text{TWF}}(x_{n+1}; \hat q)$ satisfy $1 - \alpha \leq\mathbb{P}[y_{n+1}^* \in C_{\text{TWF}}(x_{n+1}; \hat q)] < 1 - \alpha + \frac{1}{n+1}$.
\end{restatable}
\begin{proof}
The proof proceeds in the same way as \Cref{thm:LF}, using \Cref{thm:TWFEquivalence} instead of \Cref{thm:LFEquivalence}.
\end{proof}

As a final point, we note that prediction sets constructed as $C_{\text{TWF}}(x_{n+1}; \hat q) = F_{\text{TWF}}(x_{n+1}; \hat q)$ may be empty due to the use of intersection. While this is not strictly a concern as coverage is valid marginally, it can be undesirable to predict empty sets. As a fallback, TWF can predict the single most likely step under $g_\text{LF}$ or $g_\text{RF}$ (which are likely to be the same function in practice).

\newpage
\section{Additional Details of Conformal Algorithms}
\label{app:b}

In \Cref{sec:crsvp} we described how to apply the CRSVP algorithm by \citet{mortier2025conformal}, originally designed for hierarchical classification, to sequential data like agent trajectories. Here we provide a more detailed description of our implementation.

To recap, a binary tree $\mathcal{T}$ is constructed from an agent trajectory $x=(c_1, \dots, c_{\ell})$ to have root node $v_1 = [\ell]$, and leaf nodes $v_\ell, \dots, v_{2\ell-1}$ as individual steps $c_1, \dots, c_\ell$. For each calibration datapoint $(x, y^*)$, a scoring function $g_\text{CRSVP}$ is used to find the single leaf $\bar v$ most likely to contain the error $y^*$. If $\bar v$ is itself the decisive error $y^*$, then the conformal score is returned as $S_\text{CRSVP}(x, y^*) := g_\text{CRSVP}(\bar v)$. Otherwise, the tree is traversed upwards until the first node $v^*$ is reached which does contain $y^*$. Since the number of steps in each node grows exponentially as the tree is traversed, simply using the score $g_\text{CRSVP}(v^*)$ would greatly overcover the true labels. Hence, CRSVP interpolates between steps a random amount, similar to other conformal scoring functions \cite{angelopoulos2021raps}. Let $v^{*-1}$ be the node before $v^*$ on the traversal path. Then the conformal score is set as 
\begin{equation}
S_\text{CRSVP}(x, y^*) := g_\text{CRSVP}(v^*)- u\cdot(g_{\text{CRSVP}}(v^*) - g_{\text{CRSVP}}(v^{*-1})),
\end{equation}
where $u\sim \mathcal{U}(0,1)$ is noise drawn uniformly from $[0,1]$. Note that parent nodes are strict supersets of their children, which means prediction sets are nested when traversing from leaf to root \cite{GUPTA2022108496}.

After conformal calibration to get the $\frac{\lfloor(n+1)(\alpha)\rfloor}{n}$ quantile of the scores\footnote{Note that the quantile is different than VCP because we use conformity scores in CRSVP, and $1-g$ as a nonconformity score for VCP. These are merely conventions and can easily be modified.}, which we denote $\hat q$, for a test datapoint $x_{n+1}$, CRSVP again finds the most likely leaf node $\bar v$ according to $g_\text{CRSVP}$. If that node's score is below the threshold, $g_\text{CRSVP}(\bar v) < \hat q$, CRSVP recursively traverses up the tree until the threshold is surpassed at $\hat v$. Drawing noise $u$, if $g_{\text{CRSVP}}(\hat v) - u\cdot(g_{\text{CRSVP}}(\hat v) - g_{\text{CRSVP}}(\hat v^{-1}))\geq \hat q$, then $\hat v^{-1}$ is returned, otherwise $\hat v$ itself is. By construction all prediction sets generated this way are contiguous subintervals of $x_{n+1}$ and the lower bound on coverage is guaranteed (\Cref{eq:coverage-guarantee-contiguous}) \citep{mortier2025conformal}.

\newpage
\section{Implementation Details and Additional Results}
\label{app:c}

\subsection{Additional Experiment Details}
In this section we provide more detail about datasets, models, and experimental results. 

\subsubsection{Synthetic Dataset Generation}\label{app:synthetic-data}
As discussed in \cref{sec:datasets}  we construct three variants of the original GSM8k dataset that explicitly control the location of the decisive error step along the execution trajectory. Errors are first generated uniformly by selecting one step in a trajectory, and injecting instructions to the agent to fail on that step in various ways \cite{kong2025aegis, pan2025why}. We generate 1200 failed trajectories for GSM8k for each of the DyLAN and MACNET agent frameworks (and similarly for the MATH dataset). Then the trajectories are split by error position, in the first third, middle third, or final third to form the Right-, Mid-, and Left-Dense GSM8k datasets.

\subsubsection{Context-Engineered LLM Implementation}
Prior work has shown that engineering the context of LLMs and prompting them with diverse role-specific instructions can improve evaluation robustness~\citep{wu2023largelanguagemodelsdiverse}. Following this general idea, we employ multiple role prompts, inspired by ECHO~\cite{banerjee2025hierarchical}, to obtain step-level scores from different perspectives. We first use the LLM to summarize why the agent failed the task using the overall trace, then pass this information to LLMs with four different roles to generate scores. The average of these scores is the final step score. The exact prompts used for the context-engineered LLM can be found in \Cref{app: llm prompts}.

\subsubsection{Fine-tuned Scoring Function}
As described in \cref{sec:scoring function}, we fine-tuned an LLM to estimate step-level likelihoods of being the decisive error step.
\paragraph{Training data}
Due to limited data in Who\&When \citep{zhang2025agent}, we only perform fine-tuning on the MATH and GSM8k datasets with synthetically injected errors. As mentioned in \Cref{app:synthetic-data}, we generated 4{,}800 trajectories in total across datasets and MAS architectures. Of these, 4{,}000 were used for model fine-tuning, equally split across the datasets, architectures, and error locations. We then used an 85\%/15\% split for training and validation to support hyperparameter selection. The 800 trajectories not used for training were held out entirely from fine-tuning and instead were used for conformal calibration and testing with an even split. Throughout this work, we report model performance on this unseen testing split.

\paragraph{Model Training}
We fine-tuned a Qwen3-1.7B~\cite{yang2025qwen3technicalreport} LLM for 20 epochs to estimate step-wise error likelihoods in failed multi-agent trajectories. Since each trajectory had one step labeled as the decisive error (determined through error injection), we fine-tuned the model $f_\theta$ on the $\ell$-way classification task to produce a scalar score $f_\theta(c_{i, j})$ for each step $j$ in datapoint $i$, interpreted as the likelihood that step $c_{i,j}$ contains the decisive error. We write $e_{i, j}$ as the binary indicator of whether step $j$ in datapoint $i$ is the decisive error. The model parameters $\theta$ were optimized via binary cross-entropy:
\begin{equation}
\mathcal{L}(\theta)
=
-\sum_i \sum_{j=1}^{\ell_i}
\Big[
e_{i,j} \log \sigma\!\left(f_\theta(c_{i,j})\right)
+
(1-e_{i,j}) \log \!\left(1-\sigma\!\left(f_\theta(c_{i,j})\right)\right)
\Big],
\end{equation}
where $\sigma(\cdot)$ denotes the logistic sigmoid.

\begin{table}[t]
\caption{Fine-tuning configuration
}
\centering
\small
\begin{tabular}{lcccccc}
\toprule
\textbf{Model} & \textbf{Dataset}  & \textbf{Epochs} & \textbf{Batch Size} & \textbf{Precision} & \textbf{LoRA $(r,\alpha)$} & \textbf{LR} \\
\midrule
Qwen3-1.7B & GSM8k + MATH & 18 & 32 & 8-bit & (2, 8) & $2{\times}10^{-6}$ \\  
\bottomrule
\end{tabular}

\label{tab:finetune-settings}
\end{table}

We provide more training set up details in~\Cref{tab:finetune-settings}, and report the performance before and fine-tuning  the models in \Cref{tab:qwen3_finetune}. 

\paragraph{Compute Resources}
For fine-tuning the 1.7B parameter LLM we used an Nvidia GeForce RTX 5090. 20 epochs of training took roughly one day, and consumed less than 16 GB of GPU memory with batch size 32. All other calls to LLMs used commercial APIs, and we discussed the number of calls needed in \Cref{sec:efficiency}. Once scoring calls are made, performing conformal calibration is a trivial computational cost.

\begin{table}[t]
\centering
\caption{Performance of Qwen3-1.7B models before and after fine-tuning}
\begin{tabular}{lccc}
\toprule
Model & 
AUROC & AUPRC & Accuracy \\
\midrule
Pretrained
& 0.505 & 0.369 & 0.509 \\
Fine-tuned
& 0.762 & 0.382 & 0.731 \\
\bottomrule
\end{tabular}
\label{tab:qwen3_finetune}
\end{table}

\subsubsection{Rollback Experiment Implementation}
\begin{figure}[h!]
    \centering
    \includegraphics[width=0.6\linewidth, trim={45 0 0 0}, clip]{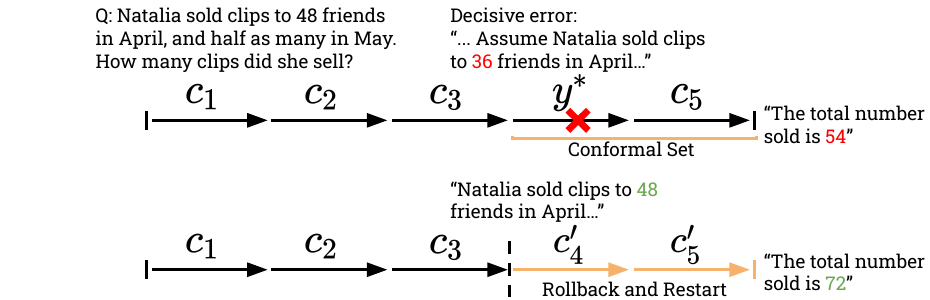}
    \caption{Example rollback scenario with text for the task, decisive error, final answer, and corrected versions after the rollback. We first generate a conformal set for the failed task, roll back the state of the MAS to the first step in the prediction set, and restart the agent with instructions to avoid the same mistakes.}
    \label{fig:rollback_texted}
    \vspace{-6pt}
\end{figure}
Figure~\ref{fig:rollback_texted} provides a concrete example of the rollback procedure. After identifying a conformal set for a failed trajectory using the LF algorithm, the system rolls back to the earliest step in the set and re-executes the task from that point on with additional context from the failed trace, allowing the MAS to correct the final outcome. Prompt details used during the restart are provided in \Cref{app: llm prompts}.

\subsection{Additional Experimental Results}
\label{app: additional_res}

\subsubsection{Scoring Aggregation Methods}

In \Cref{sec:scoring function} we discussed ways to tailor LLMs to the error attribution task and generate a step-wise error likelihood score. Step-wise scores then need to be aggregated into a set-wise score, ideally one that obeys monotonicity (\Cref{eq:g_LF_nesting}). In that section we demonstrated the normalized sum function for aggregation, $g(c_{j:k}) = \frac{1}{\ell} \sum_{i=j}^{k} \text{LLM}(c_i),$ which was used for experiments in the main text. Here we consider alternatives and show empirical results across them.

\paragraph{Max Aggregation} Another way to monotonically aggregate step-level scores is by taking the maximum score over a set:
\begin{equation}\label{eq:aggregation-max}
    g_{\text{max}}(c_{j:k}) = \max_{i\in \{j, ..., k\}} \text{LLM}(c_i) + \lambda \frac{k-j}{\ell}.
\end{equation}
Note that as opposed to summation where length normalization is included on the sum to account for trajectories of different lengths $\ell$, we do not normalize the max, since the single-step max tends to be similar across trajectories of different lengths. Instead we have included an optional length penalty term, weighted by the hyperparameter $\lambda$, to penalize sets that are longer than necessary to contain the highest scoring single step. Below, we also test adding a similar length penalty on the sum aggregation. Max aggregation should be useful when the LLM produces very peaked step-wise scores, as opposed to sum aggregation which can benefit when step-wise scores are more uniform.

\paragraph{LogSumExp Aggregation} As a third alternative, we consider a family of monotonic aggregations that interpolate between the extremes of sum and max by using the LogSumExp function:
\begin{equation}\label{eq:aggregation-lse}
    g_{\text{LSE}}(c_{j:k}) = \frac{1}{\beta}  \text{LogSumExp}(\beta \ell \cdot [\text{LLM}(c_j), ..., \text{LLM}(c_k)]) - \log \ell.
\end{equation}
This aggregation includes the inverse temperature hyperparameter $\beta$, where higher values behave more like the max function, while $\beta$ close to zero behaves like the mean function (but is monotonic, unlike the true mean function). We also add length normalization via $\ell$ to make scores more similar across trajectories with different lengths.

\paragraph{Experimental Results}

In \Cref{tab:rf_aggregators}, we compare Max, LogsumExp, and Sum aggregation with the optional length penalty using the RF conformal algorithms on the Left-Dense subset of GSM8k, across three evaluators(Naive LLM, Role-prompted, and the finetuned Qwen3-1.7B) settings and two agent architectures (DyLAN and MACNET). The primary result is that RF is largely insensitive to the details of the scoring function $g$, from the LLM design, to the aggregation function, or various hyperparameters. This means a simpler scoring function design is sufficient for use with our filtering algorithms.

\begin{table}[t]
\centering
\small
\setlength{\tabcolsep}{5pt}
\caption{Removal Rate with the Right Filtering conformal algorithm across various aggregation functions on GSM8K Left Dense dataset.}
\label{tab:rf_aggregators}
\resizebox{\textwidth}{!}{%
\begin{tabular}{llcccccc}
\toprule
\multirow{2}{*}{Aggregation} & \multirow{2}{*}{Hypers} & \multicolumn{3}{c}{DyLAN} & \multicolumn{3}{c}{MACNET} \\
\cmidrule(lr){3-5} \cmidrule(lr){6-8}
 &  & Naive LLM & Role-prompted & Qwen3-1.7B & Naive LLM & Role-prompted & Qwen3-1.7B  \\
 \midrule
\multirow{4}{*}{Sum + Length Penalty} & $\lambda=0.01$ & 0.64$\pm$0.02 & 0.63$\pm$0.02 & 0.65$\pm$0.01 & 0.59$\pm$0.01 & 0.59$\pm$0.01 & 0.62$\pm$0.02 \\
 & $\lambda=0.02$ & 0.64$\pm$0.02 & 0.63$\pm$0.02 & 0.65$\pm$0.01 & 0.59$\pm$0.01 & 0.59$\pm$0.01 & 0.62$\pm$0.02 \\
 & $\lambda=0.05$ & 0.64$\pm$0.02 & 0.63$\pm$0.02 & 0.65$\pm$0.01 & 0.59$\pm$0.01 & 0.59$\pm$0.01 & 0.62$\pm$0.02 \\
 & $\lambda=0.1$ & 0.64$\pm$0.02 & 0.63$\pm$0.02 & 0.65$\pm$0.01 & 0.59$\pm$0.01 & 0.59$\pm$0.01 & 0.62$\pm$0.01 \\
\midrule
\multirow{4}{*}{Max + Length Penalty} & $\lambda=0.01$ & 0.64$\pm$0.02 & 0.63$\pm$0.02 & 0.65$\pm$0.01 & 0.59$\pm$0.01 & 0.59$\pm$0.02 & 0.62$\pm$0.02 \\
 & $\lambda=0.02$ & 0.64$\pm$0.02 & 0.63$\pm$0.02 & 0.65$\pm$0.01 & 0.59$\pm$0.01 & 0.59$\pm$0.01 & 0.62$\pm$0.02 \\
 & $\lambda=0.05$ & 0.64$\pm$0.02 & 0.63$\pm$0.02 & 0.65$\pm$0.01 & 0.59$\pm$0.01 & 0.59$\pm$0.02 & 0.62$\pm$0.01 \\
 & $\lambda=0.1$ & 0.64$\pm$0.02 & 0.63$\pm$0.02 & 0.65$\pm$0.01 & 0.59$\pm$0.01 & 0.59$\pm$0.02 & 0.62$\pm$0.01 \\
\midrule
\multirow{3}{*}{Normalized LogSumExp} & $\beta=0.1$ & 0.64$\pm$0.02 & 0.63$\pm$0.02 & 0.65$\pm$0.01 & 0.59$\pm$0.01 & 0.59$\pm$0.02 & 0.63$\pm$0.02 \\
 & $\beta=1$ & 0.64$\pm$0.02 & 0.63$\pm$0.02 & 0.65$\pm$0.01 & 0.59$\pm$0.01 & 0.59$\pm$0.02 & 0.63$\pm$0.02 \\
 & $\beta=10$ & 0.63$\pm$0.02 & 0.63$\pm$0.02 & 0.64$\pm$0.02 & 0.60$\pm$0.01 & 0.59$\pm$0.01 & 0.62$\pm$0.01 \\
\bottomrule
\end{tabular}
}
\end{table}

\subsection{Prompts Used}
\label{app: llm prompts}
\paragraph{Naive LLM Prompts}

We use following prompt for the naive LLM scoring function.
\begin{tcolorbox}[
  colback=gray!5!white, 
  colframe=gray!75!black, 
  boxrule=0.5pt,
  left=6pt,
  right=6pt,
  top=6pt,
  bottom=6pt
]
\small\ttfamily
You are an AI assistant tasked with analyzing a segment of a multi-agent conversation.
Multiple agents are collaborating to address a user query, with the goal of resolving
the query through their collective dialogue.

Your primary task is to identify the location of the most critical mistake within the
provided segment. 

The problem to address is as follows: \{problem\}

The final correct answer for the problem is: \{answer\}

The wrong answer given at the end of the conversation is: \{wrong\_answer\}

Review the following conversation segment: \{evaluate\_content\}

Cut from the entire history:

\{chat\_segment\_content\}

Based on your analysis, please respond with ONLY a single probability value between
0 and 1, representing the estimated probability that the critical error occurs within
this conversation segment. Do not include any explanation or additional text.
\end{tcolorbox}

\paragraph{Context Engineered Prompts}
As mentioned earlier, along with the context engineered setting we use four role-based LLMs as part of the scoring function. Here, we provide the prompts used for each LLM.

\begin{tcolorbox}[
  colback=gray!5!white, 
  colframe=gray!75!black, 
  boxrule=0.5pt,
  left=6pt,
  right=6pt,
  top=6pt,
  bottom=6pt
]
\ttfamily\small

\textbf{Conservative:} \\
- Attribute an error only when explicit contradiction or logical error is visible. \\
- Assign high confidence only if the evidence is direct and quotes are exact.

\medskip 

\textbf{Liberal:} \\
- Consider potential upstream or downstream causes even if indirect. \\
- Assign moderate confidence (0.5--0.7) for plausible but not proven causes.

\medskip

\textbf{Skeptical:} \\
- Always include at least one alternative explanation and note missing evidence. \\
- Lower confidence if evidence is incomplete.

\medskip

\textbf{Pattern:} \\
- Focus on repeated reasoning or coordination issues. \\
- Identify patterns across multiple steps rather than isolated mistakes.

\end{tcolorbox}

We force the output into the following schema:

\begin{tcolorbox}[
  colback=gray!5!white,
  colframe=gray!75!black,
  fontupper=\small\ttfamily,
  title=Output Schema,
  fonttitle=\bfseries
]
\begin{verbatim}
{
  "investigation_summary": "short summary",
  "primary_conclusion": {
    "agent": "Agent-X",
    "mistake_step": "int",
    "confidence": "float (0--1)",
    "evidence": ["quoted spans or step references"],
    "reason": "brief reason",
    "alternative_explanations": ["..."]
  }
}
\end{verbatim}
\end{tcolorbox}

The final prompt for each role is as follows: 

\begin{tcolorbox}[
  colback=gray!5!white, 
  colframe=gray!75!black, 
  boxrule=0.5pt,
  left=6pt,
  right=6pt,
  top=6pt,
  bottom=6pt
]
\small\ttfamily
Problem to solve: \{problem\}

Expected correct answer: \{answer\}

Conversation segment under review:
\{evaluate\_content\}

Full chat context:
\{chat\_segment\_content\}

ROLE: \{role\} Analyst

\{role\_instructions\}

TASK: Identify whether this segment likely contains the key reasoning error that
led to the wrong final answer. Provide your analysis strictly in JSON format matching
the schema below. Quote evidence spans exactly from the text and assign a numeric
confidence between 0 and 1.

Confidence mapping:
0.0--0.2 = Very unlikely,
0.2--0.4 = Unlikely,
0.4--0.6 = Possible,
0.6--0.8 = Likely,
0.8--1.0 = Very likely.

Return ONLY valid JSON (no extra text):
\{schema\}
\end{tcolorbox}

\paragraph{Conformal Rollback Prompts}
The conformal rollback instructions provided to the MAS are as follow: 

\begin{tcolorbox}[
  colback=gray!5!white,
  colframe=gray!75!black,
  boxrule=0.5pt,
  left=6pt,
  right=6pt,
  top=6pt,
  bottom=6pt
]
\small
\ttfamily
NOTE: The following conversation history (Node 0 to Node \{cut\_index\}) contains wrong information.
Please avoid making the same mistakes if you detect the mistake:

\{wrong\_conversation\_content\}

With the following mostly correct conversation history:
\{prev\_context\}

Solve this problem step by step:
\{task\_query\}
\end{tcolorbox}



\end{document}